\pdfoutput=1

\documentclass[10pt,onecolumn,letterpaper]{article}
\usepackage[in]{fullpage}
\usepackage{natbib}

\bibpunct{(}{)}{;}{a}{,}{,}

\usepackage{amsmath,amssymb,amsfonts}
\usepackage[amsmath,thmmarks,amsthm]{ntheorem}
\usepackage{bbm} 

\usepackage{enumerate}
\usepackage{graphicx}
\usepackage{verbatim}
\usepackage{bm}
\usepackage{mathrsfs}
\usepackage{aliascnt}

\usepackage{xcolor}
\usepackage[colorlinks,
            linkcolor=blue,
            citecolor=blue,
            urlcolor=magenta,
            linktocpage,
            plainpages=false]{hyperref}
\usepackage{url}
\usepackage{nameref}

\DeclareMathOperator*{\argmin}{arg\,min}

\DeclareMathOperator*{\Prob}{\mathsf{P}}

\DeclareMathOperator*{\Probn}{\mathsf{P}_{\mathit{n}}}
\DeclareMathOperator*{\Probc}{\mathsf{\bar{P}}_{\mathit{n}}}

\DeclareMathOperator*{\conv}{\mathrm{conv}}
\DeclareMathOperator*{\interior}{\mathrm{int}}
\DeclareMathOperator*{\opvee}{\vee}
\DeclareMathOperator*{\opwedge}{\wedge}
\newcommand{\E}{\operatorname{\mathsf{E}}}
\newcommand{\real}{\mathbb{R}}

\newcommand{\nat}{\mathbb{N}}

\newcommand{\X}{\mathcal{X}}

\newcommand{\Y}{\mathcal{Y}}
\newcommand{\A}{\mathcal{A}}

\newcommand{\T}{\mathcal{T}}

\newcommand{\F}{\mathcal{F}}
\newcommand{\G}{\mathcal{G}}

\newcommand{\lossfunc}{\ell}
\newcommand{\loss}[2]{\lossfunc \bigl( #1, #2 \bigr)}

\newcommand{\lossof}[1]{\lossfunc(\cdot, #1)}
\newcommand{\lossclass}{\lossfunc \circ \F}

\renewcommand{\Pr}{\mathsf{Pr}}

\newcommand{\erm}{\hat{f}_{\mathbf{z}}}

\newcommand{\Fcond}{\F_{|_{\mathbf{X}}}}
\newcommand{\Fweak}{\F_{\succeq \varepsilon}}

\newcommand{\N}{\mathcal{N}}
\newcommand{\Entropy}{\mathcal{H}}
\newcommand{\capacity}{\mathcal{C}}
\newcommand{\bound}{V}

\definecolor{light-gray}{gray}{0.9}

\qedsymbol{\ensuremath{_\blacksquare}}

\newaliascnt{theorem}{thm}
\newtheorem{theorem}[theorem]{Theorem}
\aliascntresetthe{theorem}

\newaliascnt{corollary}{thm}
\newtheorem{corollary}[corollary]{Corollary}
\aliascntresetthe{corollary}

\newaliascnt{lemma}{thm}
\newtheorem{lemma}[lemma]{Lemma}
\aliascntresetthe{lemma}

\newaliascnt{example}{thm}

\aliascntresetthe{example}

\newaliascnt{proposition}{thm}

\aliascntresetthe{proposition}

\newaliascnt{definition}{thm}

\aliascntresetthe{definition}

\newcommand{\BlackBox}{\rule{1.5ex}{1.5ex}}  
\newenvironment{proof-sketch}{\par\noindent{\bfseries\upshape
  Proof\ sketch\ }}{\hfill\BlackBox\\[2mm]}

\title{From Stochastic Mixability to Fast Rates}

\author{
Nishant A.~Mehta \\
Australian National University \\
\and
Robert C.~Williamson \\
Australian National University and NICTA \\
}
\date{}

\newcommand{\citeposs}[1]{\citeauthor{#1}'s~(\citeyear{#1})}

\begin{document}

\maketitle 

\begin{abstract}
Empirical risk minimization (ERM) is a fundamental learning rule for statistical learning problems where the data is generated according to some unknown distribution $\Prob$ and returns a hypothesis $f$ chosen from a fixed class $\F$ with small loss $\lossfunc$. 
In the parametric setting, depending upon $(\lossfunc, \F,\Prob)$ ERM can have slow $(1/\sqrt{n})$ or fast $(1/n)$ rates of convergence of the excess risk as a function of the sample size $n$. There exist several results that give sufficient conditions for fast rates in terms of joint properties of $\lossfunc$, $\F$, and $\Prob$, such as the margin condition and the Bernstein condition. In the non-statistical prediction with expert advice setting, there is an analogous slow and fast rate phenomenon, and it is entirely characterized in terms of the mixability of the loss $\lossfunc$ (there being no role there for $\F$ or $\Prob$). The notion of stochastic mixability builds a bridge between these two models of learning, reducing to classical mixability in a special case. The present paper presents a direct proof of fast rates for ERM in terms of stochastic mixability of $(\lossfunc,\F, \Prob)$, and in so doing provides new insight into the fast-rates phenomenon. The proof exploits an old result of Kemperman on the solution to the general moment problem.  We also show a partial converse that suggests a characterization of fast rates for ERM in terms of stochastic mixability is possible.
\end{abstract}

\section{Introduction}
Recent years have unveiled central contact points between the areas of statistical and online learning. These include \citeposs{abernethy2009stochastic} unified Bregman-divergence based analysis of online convex optimization and statistical learning, the online-to-batch conversion of the exponentially weighted average forecaster (a special case of the aggregating algorithm for mixable losses) which yields the progressive mixture rule as can be seen e.g.~from the work of \cite{audibert2009fast}, and most recently \citeposs{vanerven2012mixability} injection of the concept of mixability into the statistical learning space in the form of stochastic mixability. It is this last connection that will be our departure point for this work. 

Mixability is a fundamental property of a loss that characterizes when constant regret is possible in the online learning game of prediction with expert advice \citep{vovk1998game}. Stochastic mixability is a natural adaptation of mixability to the statistical learning setting; in fact, in the special case where the function class consists of all possible functions from the input space to the prediction space, stochastic mixability is equivalent to mixability \citep{vanerven2012mixability}. Just as Vovk and coworkers (see e.g.~\citep{vovk2001competitive,kalnishkan2005weak}) have developed a rich convex geometric understanding of mixability, stochastic mixability can be understood as a sort of effective convexity.

In this work, we study the $O(\frac{1}{n})$-fast rate phenomenon in statistical learning from the perspective of stochastic mixability. Our motivation is that stochastic mixability might characterize fast rates in statistical learning. As a first step, \autoref{thm:finite-fast-rates} of this paper establishes via a rather direct argument that stochastic mixability implies an exact oracle inequality (i.e.~with leading constant 1) with a fast rate for finite function classes, and \autoref{thm:vc-type-fast-rates} extends this result to VC-type classes. 
This result can be understood as a new chapter in an evolving narrative that started with \citeposs{lee1998importance} seminal paper showing fast rates for agnostic learning with squared loss over convex function classes, and that was continued by \cite{mendelson2008obtaining} who showed that fast rates are possible for $p$-losses $(y, \hat{y}) \mapsto |y - \hat{y}|^p$ over effectively convex function classes by passing through a Bernstein condition (defined in \eqref{eqn:bernstein}).

We also show that when stochastic mixability does not hold in a certain sense (see \autoref{sec:convexity} for the precise statement), then the risk minimizer is not unique in a bad way. This is precisely the situation at the heart of the works of \cite{mendelson2008obtaining} and \cite{mendelson2002agnostic}, which show that having non-unique minimizers is symptomatic of bad geometry of the learning problem. In such situations, there are certain targets (i.e. output conditional distributions) close to the original target under which empirical risk minimization learns at a slow rate, where the guilty target depends on the sample size and the target sequence approaches the original target asymptotically. Even the best known upper bounds have constants that blow up in the case of non-unique minimizers. Thus, whereas stochastic mixability implies fast rates, a sort of converse is also true, where learning is hard in a ``neighborhood'' of statistical learning problems for which stochastic mixability does not hold. In addition, since a stochastically mixable problem's function class looks convex from the perspective of risk minimization, and since when stochastic mixability fails the function class looks non-convex from the same perspective (it has multiple well-separated minimizers), stochastic mixability \emph{characterizes} the effective convexity of the learning problem from the perspective of risk minimization.

Much of the recent work in obtaining faster learning rates in agnostic learning has taken place in settings where a Bernstein condition holds, including results based on local Rademacher complexities \citep{bartlett2005local, koltchinskii2006local}. 
The Bernstein condition appears to have first been used by \cite{bartlett2006empirical} in their analysis of empirical risk minimization; this condition is subtly different from the margin condition of \cite{mammen1999smooth} and \cite{tsybakov2004optimal}, which has been used to obtain fast rates for classification problems. \cite{lecue2011interplay} pinpoints that the difference between the two conditions is that the margin condition applies to the excess loss relative to the best predictor (not necessarily in the model class) whereas the Bernstein condition applies to the excess loss relative to the best predictor in the model class. 
Our approach in this work is complementary to the approaches of previous works, coming from a different assumption that forms a bridge to the online learning setting. Yet this assumption is related; the Bernstein condition implies stochastic mixability under a bounded losses assumption \citep{vanerven2012mixability}. 
Further understanding the connection between the Bernstein condition and stochastic mixability is an ongoing effort.

\paragraph{Contributions.}
The core contribution of this work is to show a new path to the $\tilde{O}\left( \frac{1}{n} \right)$-fast rate in statistical learning. We are not aware of previous results that show fast rates from the stochastic mixability assumption. 
Secondly, we establish intermediate learning rates that interpolate between the fast and slow rate under a weaker notion of stochastic mixability. 
Finally, we show that in a certain sense stochastic mixability characterizes the effective convexity of the statistical problem. 

In the next section we formally define the statistical problem, review stochastic mixability, and explain our high-level approach toward getting fast rates. This approach involves directly appealing to the Cram\'er-Chernoff method, from which nearly all known concentration inequalities arose in one way or another. 
In \autoref{sec:general-moment-problem}, we frame the problem of computing a particular moment of a certain excess loss random variable as a general moment problem. We sufficiently bound the optimal value of the moment, which allows for a direct application of the Cram\'er-Chernoff method. 
These results easily imply a fast rates bound for finite classes that can be extended to parametric (VC-type) function classes, as shown in \autoref{sec:fast-rates}. 
We describe in \autoref{sec:convexity} how stochastic mixability characterizes a certain notion of convexity of the statistical learning problem. 
In \autoref{sec:weak}, we extend the fast rates results to classes that obey a notion we call weak stochastic mixability. 
Finally, \autoref{sec:discussion} concludes this work with connections to related topics in statistical learning theory and a discussion of open problems.

\section{Stochastic mixability, Cram\'er-Chernoff, and ERM}
\label{sec:strategy}

\subsection{The Setting}
Let $(\lossfunc, \F, \Prob)$ be a statistical learning problem with
$\lossfunc: \Y \times \real \rightarrow \real_+$ a nonnegative loss, $\F \subset \real^\X$ a compact function class, and $\Prob$ a probability measure over $\X \times \Y$ for input space $\X$ and output/target space $\Y$. Let $Z$ be a random variable defined as $Z = (X, Y) \sim \Prob$. We assume for all $f \in \F$, $\ell(Y, f(X)) \leq \bound$ almost surely (a.s.) for some constant $\bound$. 

A probability measure $\Prob$ operates on functions and loss-composed functions as:
\begin{align*}
\Prob f = \E_{(X,Y) \sim \Prob} f(X) &&
\Prob \lossof{f} = \E_{(X,Y) \sim \Prob} \loss{Y}{f(X)} .
\end{align*}
Similarly, an empirical measure $\Probn$ associated with an $n$-sample $\mathbf{z}$, comprising $n$ iid samples $(x_1, y_1), \ldots, (x_n, y_n)$, operates on functions and loss-composed functions as:
\begin{align*}
\Probn f = \frac{1}{n} \sum_{j=1}^n f(x_j) && 
\Probn \lossof{f} = \frac{1}{n} \sum_{j=1}^n \loss{y_j}{f(x_j)} .
\end{align*}

Let $f^*$ be any function for which $\Prob \lossof{f^*} = \inf_{f \in \F} \Prob \lossof{f}$.  For each $f \in \F$ define the excess risk random variable $Z_f := \loss{Y}{f(X)} - \loss{Y}{f^*(X)}$.

We frequently work with the following two subclasses. For any $\varepsilon > 0$, define the subclasses 
\begin{align*}
\F_{\preceq \varepsilon}:= \left\{ f \in \F : \Prob Z_f \leq \varepsilon \right\} && 
\F_{\succeq \varepsilon}:= \left\{ f \in \F : \Prob Z_f \geq \varepsilon \right\} .
\end{align*}

\subsection{Stochastic mixability}
For $\eta > 0$, we say that $(\lossfunc, \F, \Prob)$ is \emph{$\eta$-stochastically mixable} if for all $f \in \F$
\begin{align}
\log \E \exp(-\eta Z_f) \leq 0 . \label{eqn:stochastic-mixability}
\end{align}
If $\eta$-stochastic mixability holds for some $\eta > 0$, then we say that $(\lossfunc, \F, \Prob)$ is \emph{stochastically mixable}. Throughout this paper it is assumed that the stochastic mixability condition holds, and we take $\eta^*$ to be the largest $\eta$ such that $\eta$-stochastic mixability holds. 
Condition \eqref{eqn:stochastic-mixability} has a rich history, beginning from the foundational thesis of \cite{li1999estimation} who studied the special case of $\eta^* = 1$ in density estimation with log loss from the perspective of information geometry. The connections that Li showed between this condition and convexity were strengthened by \cite{grunwald2011safe,grunwald2012safe} and \cite{vanerven2012mixability}.

\subsection{Cram\'er-Chernoff}
The high-level strategy taken here is to show that with high probability the empirical risk minimization algorithm (ERM) will not select a fixed hypothesis function $f$ with excess risk above $\frac{a}{n}$ for some constant $a > 0$. For each hypothesis, this guarantee will flow from the Cram\'er-Chernoff method \citep{boucheron2013concentration} by controlling the cumulant generating function (CGF) of $-Z_f$ in a particular way to yield exponential concentration. 
This control will be possible because the $\eta^*$-stochastic mixability condition implies that the CGF of $-Z_f$ takes the value 0 at some $\eta \geq \eta^*$, a fact later exploited by our key tool \autoref{thm:stochastic-mixability-concentration}.

Let $Z$ be a real-valued random variable. Applying Markov's inequality to an exponentially transformed random variable yields that, for any $\eta \geq 0$ and $t \in \real$
\begin{align}
\Pr(Z \geq t) \leq \exp(-\eta t + \log \E \exp(\eta Z)) ; \label{eqn:cramer-chernoff}
\end{align}
the inequality is non-trivial only if $t > \E Z$ and $\eta > 0$.

\subsection{Analysis of ERM}
We consider the \emph{ERM estimator} $\erm := \argmin_{f \in \F} \Probn \lossof{f}$. 
That is, given an $n$-sample $\mathbf{z}$, ERM selects any $\erm \in \F$ minimizing the empirical risk $\Probn \lossof{f}$. 
We say ERM is \emph{$\varepsilon$-good} when $\erm \in \F_{\preceq \varepsilon}$. 
In order to show that ERM is $\varepsilon$-good it is sufficient to show that for all $f \in \F \setminus \F_{\preceq \varepsilon}$ we have $\Prob Z_f > 0$. 
The goal is to show that with high probability ERM is $\varepsilon$-good, and we will do this by showing that with high probability uniformly for all $f \in \F \setminus \F_{\preceq \varepsilon}$ we have $\Probn Z_f > t$ for some slack $t > 0$ that will come in handy later.

For a real-valued random variable $X$, 
recall that the \emph{cumulant generating function} of $X$ is $\eta \mapsto \Lambda_X(\eta) := \log \E e^{\eta X}$; we allow $\Lambda_X(\eta)$ to be infinite for some $\eta > 0$.

\begin{theorem}[Cram\'er-Chernoff Control on ERM]
\label{thm:cramer-chernoff}
Let $a > 0$ and select $f$ such that $\E Z_f > 0$. Let $t < \E Z_f$. If there exists $\eta > 0$ such that $\Lambda_{-Z_f}(\eta) \leq -\frac{a}{n}$, then
\begin{align*}
\Pr \Bigl\{ \Probn \lossof{f} \leq \Probn \lossof{f^*} + t \Bigr\} \leq \exp(-a + \eta t) . 
\end{align*}
\end{theorem}
\begin{proof}
Let $Z_{f, 1}, \ldots, Z_{f, n}$ be iid copies of $Z_f$, and define the sum $S_{f,n} := \sum_{j=1}^n -Z_{f, j}$. 
Since $(-t) > \E \frac{1}{n} S_{f,n}$, then from \eqref{eqn:cramer-chernoff} we have
\begin{align*}
\Pr \left( \frac{1}{n} \sum_{j=1}^n Z_{f, j} \leq t \right)
= \Pr\left( \frac{1}{n} S_{f,n} \geq -t \right) 
&\leq \exp \left( \eta t + \log \E \exp( \eta S_{f,n} ) \right) \\
&= \exp(\eta t) \bigl( \E \exp( -\eta Z_f ) \bigr)^n .
\end{align*}
Making the replacement $\Lambda_{-Z_f}(\eta) = \log \E \exp( - \eta Z_f )$ yields
\begin{align*}
\log \Pr\left( \frac{1}{n} S_{f,n} \geq -t \right) 
\leq \eta t + n \Lambda_{-Z_f}(\eta) .
\end{align*}
By assumption, $\Lambda_{-Z_f}(\eta) \leq -\frac{a}{n}$, and so $\Pr \{ \Probn Z_f \leq t \} \leq \exp(-a + \eta t)$ as desired. 
\end{proof}

This theorem will be applied by showing that for an excess loss random variable $Z_f$ taking values in $[-1, 1]$, if for some $\eta > 0$ we have $\E \exp(-\eta Z_f) = 1$ and if $\E Z_f = \frac{a}{n}$ for some constant $a$ (that can and must depend on $n$), then $\log_{-Z_f}(\eta / 2) \leq -\frac{c \eta a}{n}$ where $c > 0$ is a universal constant. This is the nature of the next section. We then extend this result to random variables taking values in $[-\bound, \bound]$.

\section{Semi-infinite linear programming and the general moment problem}
\label{sec:general-moment-problem}

The key subproblem now is to find, for each excess loss random variable $Z_f$ with mean $\frac{a}{n}$ and $\Lambda_{-Z_f}(\eta) = 0$ (for some $\eta \geq \eta^*$), a pair of constants $\eta_0 > 0$ and $c > 0$ for which $\Lambda_{-Z_f}(\eta_0) \leq -\frac{c a}{n}$. \autoref{thm:cramer-chernoff} would then imply that ERM will prefer $f^*$ over this particular $f$ with high probability for $c a$ large enough. 
This subproblem is in fact an instance of the general moment problem, a problem on which \cite{kemperman1968general} has conducted a very nice geometric study. We now describe this problem. 

\paragraph{The general moment problem.}
Let $\mathcal{P}(\A)$ be the space of probability measures over a measurable space $\A = (\A, \mathcal{S})$. 
For real-value measurable functions $h$ and $(g_j)_{j \in [m]}$ on a measurable space $\A = (\A, \mathcal{S})$, the general moment problem is
\begin{align}
\begin{aligned}
& \inf_{\mu \in \mathcal{P}(\A)} 
& & \E_{X \sim \mu} h(X) \\
& \text{subject to}
& &     \E_{X \sim \mu} g_j(X) = y_j , \quad j \in \{1, \ldots, m\} .
\end{aligned} \label{eqn:the-general-moment-problem}
\end{align}
Let the vector-valued map $g: \A \rightarrow \real^m$ be defined in terms of coordinate functions as $(g(x))_j = g_j(x)$, and let the vector $y \in \real^m$ be equal to $(y_1, \ldots, y_m)$.

Let $D^* \subset \real^{m+1}$ be the set
\begin{align}
D^* := \biggl\{ d^* = (d_0, d_1, \ldots, d_m) \in \real^{m+1} :
  h(x) \geq d_0 + \sum_{j=1}^m d_j g_j(x) \quad \text{for all } x \in \A 
\biggr\} . \label{eqn:the-D-star}
\end{align}

Theorem 3 of \citep{kemperman1968general} states that if $y \in \interior \conv g(\A)$, the optimal value of problem \eqref{eqn:the-general-moment-problem} equals
\begin{align}
\sup \biggl\{ d_0 + \sum_{j=1}^m d_j y_j : d^* = (d_0, d_1, \ldots, d_m) \in D^* \biggr\} . \label{eqn:the-optimal-value}
\end{align}

\paragraph{Our instantiation.}

We choose $\A = [-1, 1]$, set $m = 2$ and define $h$, $(g_j)_{j \in \{1,2\}}$, and $y \in \real^2$ as:
\begin{align*}
h(x) = -e^{(\eta / 2) x}, &&
g_1(x) = x, && 
g_2(x) = e^{\eta x}, &&
y_1 = -\frac{a}{n}, &&
y_2 = 1 ,
\end{align*}
for any $\eta > 0$, $a > 0$, and $n \in \nat$.

This yields the following instantiation of the general moment problem:
\begin{subequations}
\label{eqn:our-general-moment-problem}
\begin{align}
\inf_{\mu \in \mathcal{P}([-1, 1])}
& \quad \E_{X \sim \mu} -e^{(\eta / 2) X} \label{eqn:our-general-moment-problem-a} \\
\text{subject to}
& \quad \E_{X \sim \mu} X = -\frac{a}{n} \label{eqn:our-general-moment-problem-b} \\
& \quad \E_{X \sim \mu} e^{\eta X} = 1 \label{eqn:our-general-moment-problem-c} .
\end{align}
\end{subequations}

Note that equation \eqref{eqn:the-optimal-value} from the general moment problem now instantiates to
\begin{align}
\sup \left\{ d_0 - \frac{a}{n} d_1 + d_2 : d^* = (d_0, d_1, d_2) \in D^* \right\} , \label{eqn:our-optimal-value}
\end{align}
with $D^*$ equal to the set
\begin{align}
\left\{ d^* = (d_0, d_1, d_2) \in \real^3 : 
-e^{(\eta/2) x} \geq d_0 + d_1 x + d_2 e^{\eta x} \quad \text{for all } x \in [-1, 1] \right\} . \label{eqn:our-D-star}
\end{align}

To apply Theorem 3 of \citep{kemperman1968general}, we need to ensure the condition $y \in \interior \conv g([-1, 1])$ holds. We first characterize when $y \in \conv g([-1, 1])$ holds and handle the $\interior \conv g([-1, 1])$ version after \autoref{thm:stochastic-mixability-concentration}. 

\begin{lemma}[Feasible Moments]
\label{lemma:feasible-moments}
The point $y = \left( -\frac{a}{n}, 1 \right) \in \conv g([-1, 1])$ if and only if
\begin{align}
\frac{a}{n} 
\leq \frac{e^{\eta} + e^{-\eta} - 2}{e^\eta - e^{-\eta}} 
= \frac{\cosh(\eta) - 1}{\sinh(\eta)} . \label{eqn:feasibility}
\end{align}
\end{lemma}
\begin{proof}
Let $W$ denote the convex hull of $g([-1, 1])$. We need to see if $\left( -\frac{a}{n}, 1 \right) \in W$. Note that $W$ is the convex set formed by starting with the graph of $x \mapsto e^{\eta x}$ on the domain $[-1, 1]$, including the line segment connecting this curve's endpoints $(-1, e^{-\eta})$ to $(1, e^{\eta x})$, and including all of the points below this line segment but above the aforementioned graph. That is, $W$ is precisely the set
\begin{align*}
W := \left\{ (x,y) \in \real^2 :  e^{\eta x} \leq y \leq \frac{e^{\eta} + e^{-\eta}}{2} + \frac{e^{\eta} - e^{-\eta}}{2} x ,
\,\, \forall x \in [-1, 1]  \right\} .
\end{align*}

It remains to check that $1$ is sandwiched between the lower and upper bounds at $x = -\frac{a}{n}$. Clearly the lower bound holds. Simple algebra shows that the upper bound is equivalent to condition \eqref{eqn:feasibility}.
\end{proof}

Note that if \eqref{eqn:feasibility} does not hold, then the semi-infinite linear program \eqref{eqn:our-general-moment-problem} is infeasible; infeasibility in turn implies that such an excess loss random variable cannot exist. Thus, we need not worry about whether \eqref{eqn:feasibility} holds; it holds for \emph{any} excess loss random variable satisfying constraints \eqref{eqn:our-general-moment-problem-b} and \eqref{eqn:our-general-moment-problem-c}.

The following theorem is a key technical result for using stochastic mixability to control the CGF. The proof is long and can be found in \autoref{sec:appendix-stochmix}.
\begin{theorem}[Stochastic Mixability Concentration]
\label{thm:stochastic-mixability-concentration}
Let $f$ be an element of $\F$ with $Z_f$ taking values in $[-1, 1]$, $n \in \nat$, $\E Z_f = \frac{a}{n}$ for some $a > 0$, and $\Lambda_{-Z_f}(\eta) = 0$ for some $\eta > 0$. If
\begin{align}
\frac{a}{n} < \frac{e^{\eta} + e^{-\eta} - 2}{e^\eta - e^{-\eta}} , \label{eqn:interior-point}
\end{align}
then
\begin{align*}
\E e^{(\eta / 2) (-Z_f)} \leq
\begin{cases}
1 - \frac{0.18 \eta a}{n} & \text{if } \eta \leq 1 \\
1 - \frac{0.21 a}{n} & \text{if } \eta > 1 .
\end{cases}
\end{align*}
Therefore, $\E e^{(\eta / 2) (-Z_f)} \leq 1 - \frac{0.18 (\eta \opwedge 1) a}{n}$.
\end{theorem}

Note that since $\log(1 - x) \leq -x$ when $x < 1$, we have $\Lambda_{-Z_f}(\eta / 2) \leq -\frac{0.18 (\eta \opwedge 1) a}{n}$. 

In order to apply \autoref{thm:stochastic-mixability-concentration}, we need \eqref{eqn:interior-point} to hold, but only \eqref{eqn:feasibility} is guaranteed to hold. The corner case is if \eqref{eqn:feasibility} holds with equality. However, observe that one can always approximate the random variable $X$ by a perturbed version $X'$ which has nearly identical mean $a' \approx a$ and a nearly identical $\eta' \approx \eta$ for which $\E_{X' \sim \mu'} e^{\eta' X'} = 1$, and yet the inequality in \eqref{eqn:feasibility} is strict. Later, in the proof of \autoref{thm:finite-fast-rates}, for any random variable that required perturbation to satisfy the interior condition \eqref{eqn:interior-point}, we implicitly apply the analysis to the perturbed version, show that ERM would not pick the (slightly different) function corresponding to the perturbed version, and use the closeness of the two functions to show that ERM also would not pick the original function.

We now present a necessary extension for the case of losses with range $[0, \bound]$.

\begin{lemma}[Bounded Losses]
\label{lemma:bounded-losses}
Let $g_1(x) = x$ and $y_2 = 1$ be common settings for the following two problems. 
The instantiation of problem \eqref{eqn:the-general-moment-problem} with $\A = [-\bound, \bound]$, $h(x) = -e^{(\eta / 2) x}$, $g_2(x) = e^{\eta x}$, and $y_1 = -\frac{a}{n}$ has the same optimal value as the instantiation of problem \eqref{eqn:the-general-moment-problem} with $\A = [-1, 1]$, $h(x) = -e^{(\bound \eta / 2) x}$, $g_2(x) = e^{ (\bound \eta) x}$, and $y_1 = -\frac{a / \bound}{n}$. 
\end{lemma}

\begin{proof}
Let $X$ be a random variable taking values in $[-\bound, \bound]$ with mean $-\frac{a}{n}$ and $\E e^{\eta X} = 1$, and let $Y$ be a random variable taking values in $[-1, 1]$ with mean $-\frac{a / \bound}{n}$ and $\E e^{(\bound \eta) Y} = 1$. 
Consider a random variable $\tilde{X}$ that is a $\frac{1}{\bound}$-scaled independent copy of $X$; observe that $\E \tilde{X} = -\frac{a / \bound}{n}$ and $\E e^{(\bound \eta) \tilde{X}} = 1$. 
Let the maximal possible value of $\E e^{(\eta / 2) X}$ be $b_X$, and let the maximal possible value of $\E e^{(\bound \eta / 2) Y}$ be $b_Y$. We claim that $b_X = b_Y$. Let $X$ be a random variable with a distribution that maximizes $\E e^{(\eta / 2) X}$ subject to the previously stated constraints on $X$. Since $\tilde{X}$ satisfies $\E e^{(\bound \eta / 2) \tilde{X}} = b_X$, setting $Y = \tilde{X}$ shows that in fact $b_Y \geq b_X$. A symmetric argument (starting with $Y$ and passing to some $\tilde{Y} = \bound Y$) implies that $b_X \geq b_Y$.
\end{proof}

\section{Fast rates}
\label{sec:fast-rates}

We now show how the above results can be used to obtain an exact oracle inequality with a fast rate. We first present a result for finite classes and then present a result for various parametric classes, including VC-type classes (classes with logarithmic universal metric entropy), VC classes, and classes with polynomial uniform $L_1$-bracketing numbers.

\begin{theorem}[Finite Classes Exact Oracle Inequality]
\label{thm:finite-fast-rates} 
Let $(\lossfunc, \F, \Prob)$ be $\eta^*$-stochastically mixable, where $|\F|  = N$, $\lossfunc$ is a nonnegative loss, and $\sup_{f \in \F} \loss{Y}{f(X)} \leq \bound$ a.s.~for a constant $\bound$. Then for all $n \geq 1$, with probability at least $1 - \delta$
\begin{align*}
\Prob \lossof{\erm} 
\leq \Prob \lossof{f^*} 
       + \frac{6 \max\left\{ \bound, \frac{1}{\eta^*} \right\} \left( \log \frac{1}{\delta} + \log N 
\right)}{n} .
\end{align*}
\end{theorem}
\begin{proof}
Throughout this proof, let $\gamma_n = \frac{a}{n}$ where $a$ is a constant that varies throughout the proof. 
For any $a > 0$, recall that $\F_{\succeq \gamma_n}$ is the subclass of $\F$ for which the excess risk is at least $\gamma_n$. 
For each $\eta > 0$, let $\F_{\succeq \gamma_n}^{(\eta)} \subset \F_{\succeq \gamma_n}$ correspond to those functions in $\F_{\succeq \gamma_n}$ for which $\eta$ is the largest constant such that 
$\E \exp(-\eta Z_f) = 1$. 
Let $\F_{\succeq \gamma_n}^{\mathrm{hyper}} \subset \F_{\succeq \gamma_n}$ correspond to functions $f$ in $\F_{\succeq \gamma_n}$ for which $\lim_{\eta \rightarrow \infty} \E \exp(-\eta Z_f) < 1$. 
Clearly, $\F_{\succeq \gamma_n} = \left( \bigcup_{\eta \in [\eta^*, \infty)} \F_{\succeq \gamma_n}^{(\eta)} \right) \cup \F_{\succeq \gamma_n}^{\mathrm{hyper}}$. 
The excess loss random variables corresponding to elements $f \in \F_{\succeq \gamma_n}^{\mathrm{hyper}}$ are ``hyper-concentrated'' in the sense that they are infinitely stochastically mixable. However, \autoref{lemma:hyper-concentrated} shows that for each hyper-concentrated excess loss random variable $Z_f$, there exists another excess loss random variable $Z'_f$ with mean arbitrarily close to that of $Z_f$, with $\E \exp(-\eta Z'_f) =1$ for some arbitrarily large but finite $\eta$, and with $Z'_f \leq Z_f$ with probability 1. The last property implies that the empirical risk of $Z'_f$ is no greater than the empirical risk of $Z_f$, and hence for each hyper-concentrated $Z_f$ it is sufficient (from the perspective of ERM's behavior) to study a corresponding $Z'_f$. From here on out, we implicitly make this replacement in $\F_{\succeq \gamma_n}$ itself, so that we now have $\F_{\succeq \gamma_n} = \bigcup_{\eta \in [\eta^*, \infty)} \F_{\succeq \gamma_n}^{(\eta)}$.

Consider an arbitrary $a > 0$. For some fixed $\eta \in [\eta^*, \infty)$ for which $| \F_{\succeq \gamma_n}^{(\eta)} | > 0$, consider the subclass $\F_{\succeq \gamma_n}^{(\eta)}$. 
Individually for each such function, we will apply \autoref{thm:cramer-chernoff} as follows. From \autoref{lemma:bounded-losses}, 
we have $\Lambda_{-Z_f}(\eta / 2) = \Lambda_{-\frac{1}{\bound} Z_f}(\bound \eta / 2)$. From \autoref{thm:stochastic-mixability-concentration}, the latter is at most
\begin{align*}
-\frac{0.18 (\bound \eta \operatorname{\wedge} 1) (a / \bound)}{n} 
= -\frac{0.18 \eta a}{(\bound \eta \operatorname{\vee} 1) n} .
\end{align*}
Hence, \autoref{thm:cramer-chernoff} with $t = 0$ and the $\eta$ from the Theorem taken to be $\eta / 2$ implies that the probability of the event $\Probn \lossof{f} \leq \Probn \lossof{f^*}$ is at most
\begin{align*}
\exp\left( -0.18 \frac{\eta}{\bound \eta \operatorname{\vee} 1} a \right) .
\end{align*}

Applying the union bound over all of $\F_{\succeq \gamma_n}$, we conclude that
\begin{align*}
\Pr \left\{ 
\exists f \in \F_{\succeq \gamma_n} : \Probn \lossof{f} \leq \Probn \lossof{f^*} 
\right\} 
\leq N \exp \left( -\eta^* \left( \frac{0.18 a}{\bound \eta^* \operatorname{\vee} 1} \right) \right) .
\end{align*}

Now, recalling that ERM selects hypotheses purely based on their empirical risk, from inversion it holds that with probability at least $1 - \delta$, ERM will not select any hypothesis whose excess risk is at least 
\begin{align*}
\frac{6 \max\left\{ \bound, \frac{1}{\eta^*} \right\} \left( \log \frac{1}{\delta} + \log N 
\right)}{n} .
\end{align*}
\end{proof}

Before presenting the result for VC-type classes, we require some definitions. 
For a pseudometric space $(\G, d)$, for any $\varepsilon > 0$, let $\N(\varepsilon, \G, d)$ be the $\varepsilon$-covering number of $(\G, d)$; that is, $\N(\varepsilon, \G, d)$ is the minimal number of balls of radius $\varepsilon$ needed to cover $\G$. We will further constrain the cover (the set of centers of the balls) to be a subset of $\G$, so that the cover is a proper cover, thus ensuring that the stochastic mixability assumption transfers to any (proper) cover of $\F$. Note that the ``proper'' requirement at most doubles the constant $K$ below, as can be seen from an argument of \citet[Lemma 2.1]{vidyasagar2002learning}.

We now state a localization-based result that allows us to extend the result for finite classes to VC-type classes. Although the localization result can be obtained by combining standard techniques,\footnote{See e.g.~the techniques of \cite{massart2006risk} and (for Step 3 of our proof of \autoref{thm:local-analysis})) equation (3.17) of \cite{koltchinskii2011oracle}.} we could not find this particular result in the literature. 
Below, an $\varepsilon$-net $\F_\varepsilon$ of a set $\F$ is a subset of $\F$ such that $\F$ is contained in the union of the balls of radius $\varepsilon$ with centers in $\F_\varepsilon$.
\begin{theorem}
\label{thm:local-union} 
Let $\F$ be a separable function class whose functions have range bounded in $[0, \bound]$ and for which, for a constant $K \geq 1$, for each $u \in (0, K]$ the $L_2(\Prob)$ covering numbers are bounded as
\begin{align}
\N(u, \F, L_2(\Prob)) \leq \left( \frac{K}{u} \right)^\capacity . \label{eqn:poly-covering}
\end{align}
Suppose $\F_\varepsilon$ is a minimal $\varepsilon$-net for $\F$ in the $L_2(\Prob)$ norm, with $\varepsilon = \frac{1}{n}$. 
Denote by $\pi: \F \rightarrow \F_\varepsilon$ an $L_2(\Prob)$-metric projection from $\F$ to $\F_\varepsilon$. 
Then, provided that $\delta \leq \frac{1}{2}$, the probability that there exists $f \in \F$ such that 
\begin{align*}
\Probn f 
< \Probn (\pi(f)) 
- \frac{\bound}{n} \left( 
  1080 \capacity \log(2 K n) 
  + 90 \sqrt{\left( \log \frac{1}{\delta} \right) \capacity \log(2 K n)} 
  + \log \frac{e}{\delta} 
\right)
\end{align*}
is at most $\delta$. 
\end{theorem}

The proof is long and distracts from the main point of this paper, and so it is presented in \autoref{sec:appendix-vc-type}. Using \autoref{thm:local-union} along with much of the proof for the finite classes case, we can prove the following fast rates result for VC-type classes. The proof can be found in \autoref{sec:appendix-vc-type}. 
Below, we denote the loss-composed version of a function class $\F$ as $\lossclass := \{\lossof{f} : f \in \F\}$.

\begin{theorem}[VC-Type Classes Exact Oracle Inequality]
\label{thm:vc-type-fast-rates}
Let $(\lossfunc, \F, \Prob)$ be $\eta^*$-stochastically mixable with $\lossclass$ separable, where, 
for a constant $K \geq 1$, for each $\varepsilon \in (0, K]$ we have $\N(\lossclass, L_2(\Prob), \varepsilon) \leq \left( \frac{K}{\varepsilon} \right)^\capacity$, and $\sup_{f \in \F} \loss{Y}{f(X)} \leq \bound$ a.s.~for a constant $\bound \geq 1$. Then for all $n \geq 5$ and $\delta \leq \frac{1}{2}$, with probability at least $1 - \delta$
\begin{align*}
\Prob \lossof{\erm} 
\leq \Prob \lossof{f^*} 
+ \frac{1}{n} \max\left\{ 
\begin{array}{c}
  8 \max\left\{ \bound, \frac{1}{\eta^*} \right\} 
  \left( \capacity \log(K n) + \log \frac{2}{\delta} \right) , \\
  2 \bound \left( 
    1080 \capacity \log(2 K n) 
    + 90 \sqrt{\left( \log \frac{2}{\delta} \right) \capacity \log(2 K n)} 
    + \log \frac{2 e}{\delta} 
  \right)
\end{array}
\right\} + \frac{1}{n} .
\end{align*}
\end{theorem}

\section{Characterizing convexity from the perspective of risk minimization}
\label{sec:convexity}

In the following, when we say $(\lossfunc, \F, \Prob)$ has a \emph{unique minimizer} we mean that any two minimizers $f^*_1, f^*_2$ of $\Prob \lossof{f}$ over $\F$ satisfy $\loss{Y}{f^*_1(X)} = \loss{Y}{f^*_2(X)}$ a.s. 
We say the excess loss class $\{ \lossof{f} - \lossof{f^*} : f \in \F\}$ satisfies a \emph{$(\beta, B)$-Bernstein condition with respect to $\Prob$} 
for some $B > 0$ and $0 < \beta \leq 1$ if, for all $f \in \F$:
\begin{align}
\Prob \bigl( \lossof{f} - \lossof{f^*} \bigr)^2 \leq B \left( \Prob \bigl( \lossof{f} - \lossof{f^*} \bigr) \right)^\beta . \label{eqn:bernstein}
\end{align}

It already is known that the stochastic mixability condition guarantees that there is a unique minimizer \citep{vanerven2012mixability}; this is a simple consequence of Jensen's inequality. This leaves open the question: if stochastic mixability does not hold, are there necessarily non-unique minimizers? We show that in a certain sense this is indeed the case, in bad way: the set of minimizers will be a disconnected set.

For any $\varepsilon > 0$, define $\mathcal{G}_\varepsilon$ as the class
\begin{align}
\G_\varepsilon := 
\{f^*\} \cup \left\{ f \in \F : \|f - f^*\|_{L_1(\Prob)} \geq \varepsilon \right\} , \label{eqn:G-epsilon}
\end{align}
where in case there are multiple minimizers in $\F$ we arbitrarily select one of them as $f^*$. Since we assume that $\F$ is compact and $G_\varepsilon \setminus \{f^*\}$ is equal to $\F$ minus an open set homeomorphic to the unit $L_1(\Prob)$ ball, $\G_\varepsilon \setminus \{f^*\}$ is also compact. 
\begin{theorem}[Non-Unique Minimizers]
\label{thm:non-unique}
Suppose there exists some $\varepsilon > 0$ such that $\G_\varepsilon$ is not stochastically mixable. Then there are minimizers $f^*_1, f^*_2 \in \F$ of $\Prob \lossof{f}$ over $\F$ such that it is \emph{not} the case that $\loss{Y}{f^*_1(X)} = \loss{Y}{f^*_2(X)}$ a.s.
\end{theorem}
\begin{proof}
Select $\varepsilon > 0$ such that $\G_\varepsilon$ is not stochastically mixable. Consider some fixed $\eta > 0$. Since $\G_\varepsilon$ is not stochastically mixable and hence not $\eta$-stochastically mixable, there exists $f_\eta \in \G_\varepsilon$ such that $\Lambda_{-Z_{f_\eta}}(\eta) > 0$. 

We claim that there exists $\eta' \in (0, \eta)$ such that $\Lambda_{-Z_{f_\eta}}(\eta') = 0$. If not, then 
$\lim_{\eta \downarrow 0} \frac{\Lambda_{-Z_{f_\eta}}(\eta) - \Lambda_{-Z_{f_\eta}(0)}}{\eta} > 0$ and hence $\Lambda'_{-Z_{f_\eta}}(0) > 0$, which since $\Lambda'_{-Z_{f_\eta}}(0) = \E (-Z_{f_\eta})$ implies that $\E Z_{f_\eta} < 0$, a contradiction! Hence, indeed $\exists \eta' \in (0, \eta)$ such that $\Lambda_{-Z_{f_\eta}}(\eta') = 0$. 
Now, the Feasible Moments Lemma (\autoref{lemma:feasible-moments}) implies that $\E Z_{f_\eta} \leq \frac{\cosh(\eta') - 1}{\sinh(\eta')}$; 
for $\eta' \geq 0$ the RHS has the tight upper bound $\frac{\eta'}{2}$ since the derivative of $\frac{\eta'}{2} - \frac{\cosh(\eta') - 1}{\sinh(\eta')}$ is the nonnegative function $\frac{1}{2} \tanh^2(\eta'/2)$ and 
$\left( \frac{\eta'}{2} - \frac{\cosh(\eta') - 1}{\sinh(\eta')} \right)|_{\eta' = 0} = 0$. 

Thus, as $\eta \rightarrow 0$ we have $\E Z_{f_\eta} \rightarrow 0$. Since $\G_\varepsilon \setminus \{f^*\}$ is compact, we can take a positive decreasing sequence $\eta_1, \eta_2, \ldots$ approaching $0$, corresponding to a sequence of $(f_{\eta_j})_j \subset \G_\varepsilon \setminus \{f^*\}$ with limit point $g^* \in \G_\varepsilon \setminus \{f^*\}$ for which $\E Z_{g^*} = 0$, and so there is a risk minimizer in $\G_\varepsilon \setminus \{f^*\}$.
\end{proof}

\paragraph{The implications of having non-unique risk minimizers}
In the case of non-unique risk minimizers, \cite{mendelson2008lower} showed that for $p$-losses $(y, \hat{y}) \mapsto |y - \hat{y}|^p$ with $p \in [2, \infty)$ there is an $n$-indexed sequence of probability measures $(\Prob^{(n)})_n$ approaching the true probability measure as $n \rightarrow \infty$ such that, for each $n$, ERM learns at a slow rate under sample size $n$ when the true distribution is $\Prob^{(n)}$. This behavior is a consequence of the statistical learning problem's poor geometry: there are multiple minimizers and the set of minimizers is not even connected.

Furthermore, in this case, the best known fast rate upper bounds (see \citep{mendelson2008obtaining} and \citep{mendelson2002agnostic}) have a multiplicative constant that approaches $\infty$ as the target probability measure approaches a probability measure for which there are non-unique minimizers. The reason for the poor upper bounds in this case is that the constant $B$ in the Bernstein condition explodes, and the upper bounds rely upon the Bernstein condition.

\section{Weak stochastic mixability}
\label{sec:weak}

For some $\kappa \in [0, 1]$, we say $(\lossfunc, \F, \Prob)$ is ($\kappa, \eta_0)$-weakly stochastically mixable if, for every $\varepsilon > 0$, for all $f \in \{f^*\} \cup \Fweak$ 
\begin{align}
\log \E \exp(-\eta_\varepsilon Z_f) \leq 0 , \label{eqn:weak-stochastic-mixability}
\end{align}
with $\eta_\varepsilon := \eta_0 \varepsilon^{1 - \kappa}$. This concept was introduced by \cite{vanerven2012mixability} without a name.

Suppose that some fixed function has excess risk $a = \varepsilon$. Then, roughly, with high probability ERM does not make a mistake provided that $a \eta_a = \frac{1}{n}$, i.e.~when $\varepsilon \cdot \eta_0 \varepsilon^{1 - \kappa} = \frac{1}{n}$ and hence when $\varepsilon = (\eta_0 n)^{-1 / (2 - \kappa)}$. 
By modifying the proof of the finite classes result (\autoref{thm:finite-fast-rates}) to consider all functions in the subclass $\F_{\succeq \gamma_n}$ for 
$\gamma_n = (\eta_0 n)^{-1 / (2 - \kappa)}$, we have the following corollary of \autoref{thm:finite-fast-rates}.
\begin{corollary}
\label{cor:weak-intermediate-rates}
Let $(\lossfunc, \F, \Prob)$ be $(\kappa, \eta_0)$-weakly stochastically mixable for some $\kappa \in [0, 1]$, where $|\F|  = N$, $\lossfunc$ is a nonnegative loss, and $\sup_{f \in \F} \loss{Y}{f(X)} \leq \bound$ a.s.~for a constant $\bound$. Then for any $n \geq \frac{1}{\eta_0} \bound^{(1 - \kappa) / (2 - \kappa)}$, with probability at least $1 - \delta$
\begin{align*}
\Prob \lossof{\erm} 
\leq \Prob \lossof{f^*} 
       + \frac{6 \left( \log \frac{1}{\delta} + \log N \right)}{(\eta_0 n)^{1 / (2 - \kappa)}} .
\end{align*}
\end{corollary}

It is straightforward to achieve a similar result for VC-type classes, where the $\varepsilon$ in the $\varepsilon$-net can still be taken at the resolution $\frac{1}{n}$, but we need only apply the analysis to the subclass of functions with excess risk at least $(\eta_0 n)^{-1 / (2 - \kappa)}$.

\section{Discussion}
\label{sec:discussion}

We have shown that stochastic mixability implies fast rates for VC-type classes, using a direct argument based on the Cram\'er-Chernoff method and sufficient control of the optimal value of a certain instance of the general moment problem. The approach is amenable to localization in that the analysis separately controls the probability of large deviations for individual elements of $\F$. It was therefore straightforward to extend the result for finite classes to VC-type classes. 
An important open problem is to extend the results presented here for VC-type classes to results for nonparametric classes with polynomial metric entropy, and moreover, to achieve rates similar to those obtained for these classes under the Bernstein condition. 

There are still some unanswered questions with regards to the connection between the Bernstein condition and stochastic mixability. \cite{vanerven2012mixability} showed that for bounded losses the Bernstein condition implies stochastic mixability. Therefore, when starting from a Bernstein condition, \autoref{thm:finite-fast-rates} offers a different path to fast rates. An open problem is to settle the question of whether the Bernstein condition and stochastic mixability are equivalent. Previous results \citep{vanerven2012mixability} suggest that the stochastic mixability does imply a Bernstein condition, but the proof was non-constructive, and it relied upon a bounded losses assumption. It is well known (and easy to see) that both stochastic mixability and the Bernstein condition hold only if there is a unique minimizer. \autoref{thm:non-unique} shows in a certain sense that if stochastic mixability does not hold, then there cannot be a unique minimizer. Is the same true when the Bernstein condition fails to hold? 
Regardless of whether stochastic mixability is equivalent to the Bernstein condition, the direct argument presented here and the connection to classical mixability, which does characterize constant regret in the simpler non-stochastic setting, motivates further study of stochastic mixability.

Finally, it would be of great interest to discard the bounded losses assumption. Ignoring the dependence of the metric entropy on the maximum possible loss, the upper bound on the loss $\bound$ enters the final bound through the difficulty of controlling the minimum value of $u_\eta(-1)$ when $\eta$ is large (see the proof of \autoref{thm:stochastic-mixability-concentration}). From extensive experiments with a grid-approximation linear program, we have observed that the worst (CGF-wise) random variables for fixed negative mean and fixed optimal stochastic mixability constant are those which place very little probability mass at $-\bound$ and most of the probability mass at a small positive number that scales with the mean. These random variables correspond to functions that with low probability beat $f^*$ by a large (loss) margin but with high probability have slightly higher loss than $f^*$. 
It would be useful to understand if this exotic behavior is a real concern and, if not, find a simple, mild condition on the moments that rules it out.

\subsubsection*{Acknowledgments}
RCW thanks Tim van Erven  for the initial discussions around the Cram\'er-Chernoff method during his visit to Canberra in 2013 and for his gracious permission to proceed with the present paper without him as an author, and both authors thank him for the further enormously helpful spotting of a serious error in our original proof for fast rates for VC-type classes. This work was supported by the Australian Research Council (NAM and RCW) and NICTA (RCW). NICTA is funded by the Australian Government through the Department of Communications and the Australian Research Council through the ICT Centre of Excellence program.

\bibliography{arxiv_stoch_mix_fast}

\appendix

\section{Proof of \autoref*{thm:stochastic-mixability-concentration}}
\label{sec:appendix-stochmix}

\begin{proof}[Proof of \autoref{thm:stochastic-mixability-concentration}]

From Theorem 3 of \cite{kemperman1968general}, if the moment values vector $\left( -\frac{a}{n}, 1 \right)$ is in $\interior \conv g([-1, 1])$, the optimal objective value of problem \eqref{eqn:our-general-moment-problem} is equal to
\begin{align}
\sup \left\{ d_0 - \frac{a}{n} d_1 + d_2 : d^* = (d_0, d_1, d_2) \in D^* \right\} . \label{eqn:d-objective}
\end{align}
From the Feasible Moments Lemma (Lemma 2 in the paper), we see that \eqref{eqn:interior-point} corresponds to the interior point condition.

Since we assume the interior point condition is satisfied, any $d^* \in D^*$ provides a lower bound on the optimal value of \eqref{eqn:our-general-moment-problem}, and hence after negation provides an upper bound on the problem with same moment constraints and the objective $\sup \E e^{(\eta / 2) X}$; this is precisely what we are after.

We therefore focus on picking a good $d^* = (d_0, d_1, d_2) \in \real^3$. The inequality condition in \eqref{eqn:our-D-star} is now
\begin{align*}
-e^{(\eta/2) x} \geq d_0 + d_1 x + d_2 e^{\eta x} \quad \text{for all } x \in [-1, 1] .
\end{align*}
In particular, this inequality must hold at $x = 0$, yielding the constraint $-1 \geq d_0 + d_2$. We now change variables to $c_0 = -d_0$, $c_1 = -d_1 / \eta$,\footnote{We scale by $\eta$ here because we are chasing a certain $\eta$-dependent rate.} and $c_2 = -d_2$, yielding the inequality condition
\begin{align*}
u_\eta(x) := -e^{(\eta/2) x} + c_0 + c_2 e^{\eta x} + \eta c_1 x \geq 0 .
\end{align*}
Now the condition at $x = 0$ implies that $c_0 + c_2 = 1$, and so we make the replacement
\begin{align}
c_0 = 1 - c_2 , \label{eqn:c_0-constraint}
\end{align}
in the definition of $u_\eta(x)$, yielding the inequality
\begin{align*}
u_\eta(x) = 1 + c_2 (e^{\eta x} - 1) -e^{(\eta/2) x} + \eta c_1 x \geq 0 .
\end{align*}

\subsubsection*{Constraints from the local minimum at $\mathbf{0}$}

Since $u_\eta(0) = 0$, we need $x = 0$ to be a local minimum of $u$, and so we require both conditions
\begin{enumerate}[(a)]
\item $u'(0) = 0$
\item $u''(0) \geq 0$
\end{enumerate}
to hold since otherwise there exists some small $\varepsilon > 0$ such that either $u_\eta(\varepsilon) < 0$ or $u_\eta(-\varepsilon) < 0$. 

For (a), we compute
\begin{align*}
u'(x) = \eta c_2 e^{\eta x} -\frac{\eta}{2} e^{(\eta/2) x} + \eta c_1 .
\end{align*}
Since we require $u'(0) = 0$, we pick up the constraint
\begin{align*}
\eta \left( c_2 -\frac{1}{2} + c_1 \right) = 0 ,
\end{align*}
and since $\eta > 0$ by assumption, we have 
\begin{align}
c_1 = \frac{1}{2} - c_2 . \label{eqn:c_1-constraint}
\end{align}
Thus, we can eliminate $c_1$ from $u_\eta(x)$:
\begin{align*}
u_\eta(x) = 1 + c_2 (e^{\eta x} - 1) -e^{(\eta/2) x} + \eta \left( \frac{1}{2} - c_2 \right) x \geq 0 .
\end{align*}

For (b), it is sufficient to have $u''(0) > 0$. Observe that
\begin{align*}
u''(x) = \eta^2 c_2 e^{\eta x} - \frac{\eta^2}{4} e^{(\eta / 2) x} ,
\end{align*}
so that $u''(0) = \eta^2 \left( c_2 - \frac{1}{4} \right)$, and hence for
\begin{align}
c_2 > \frac{1}{4} \label{eqn:c_2-one-fourth-lower-bound}
\end{align}
we have $u''(0) > 0$. 

Thus far, we have picked up the constraints \eqref{eqn:c_0-constraint}, \eqref{eqn:c_1-constraint}, and \eqref{eqn:c_2-one-fourth-lower-bound}.

\subsubsection*{The other minima of $u_\eta(x)$}

Now, observe that $u'(x)$ has at most two roots, because with the substitution $y = e^{(\eta / 2) x}$, we have
\begin{align*}
u'(x) = \eta c_2 y^2 - \frac{\eta}{2} y + \eta \left( \frac{1}{2} - c_2 \right) ,
\end{align*}
which is a quadratic equation in $y$ with two roots:
\begin{align*}
y = \left\{ 1, \frac{1 - 2 c_2}{2 c_2} \right\} 
\quad \Rightarrow \quad 
x = \left\{ 0, \frac{2}{\eta} \log \frac{1 - 2 c_2}{2 c_2} \right\} .
\end{align*}

Now, since we take $c_2 > \frac{1}{4}$ and since the second root is negative, we know that $u$ is increasing on $[0, 1]$ (and we already knew that $u_\eta(0) = 0$). It remains to find conditions on $c_2$ such that $u_\eta(-1) \geq 0$ because that implies that $u_\eta(x) \geq 0$ for all $x \in [-1, 0]$. We consider the case $\eta \leq 1$ and $\eta > 1$ separately.

In either case, we need to check the nonnegativity of
\begin{align*}
u_\eta(-1) 
&= 1 + c_2 (e^{-\eta} - 1) -e^{-(\eta/2)} - \eta \left( \frac{1}{2} - c_2 \right) \\
&= \left( 1 - \frac{\eta}{2} \right)  - e^{-(\eta/2)}  
      + c_2 \left( e^{-\eta} - (1 - \eta) \right) .
\end{align*}

\paragraph{Case $\bm{\eta \leq 1}$:}
We observe that $u_\eta(-1) = 0$ when $\eta = 0$. Now, we will see what constraints on $c_2$ guarantee that $\frac{d}{d \eta} u_\eta(-1) \geq 0$ for $\eta \in [0, 1]$. We \emph{want}
\begin{align*}
\frac{d}{d \eta} u_\eta(-1) = -c_2 e^{-\eta} + \frac{1}{2} e^{-\eta / 2} - \frac{1}{2} +  c_2 \geq 0
\end{align*}
which is equivalent to the condition
\begin{align*}
c_2 \geq \frac{1}{2} \left( \frac{1 - e^{-\eta / 2}}{1 - e^{-\eta}} \right) .
\end{align*}
The RHS is increasing in $\eta$, and so we need only consider $\eta = 1$, yielding the bound
\begin{align*}
c_2 \geq \frac{1}{2} \frac{e - \sqrt{e}}{e - 1} = 0.3112\ldots ,
\end{align*}
and so if $c_2 \geq 0.32$, then $u_\eta(-1) \geq 0$ as desired.

\paragraph{Case $\bm{\eta > 1}$:}
Let $c_2 = \frac{1}{2} - \frac{\alpha}{\eta}$ for some $\alpha \geq 0$. With this substitution, we have
\begin{align*}
u_\eta(-1) 
&= 1 + c_2 (e^{-\eta} - 1) -e^{-(\eta/2)} - \eta \left( \frac{1}{2} - c_2 \right) \\
&= 1 + \left( \frac{1}{2} - \frac{\alpha}{\eta} \right) (e^{-\eta} - 1) -e^{-(\eta/2)} - \alpha \\
&= \left( \frac{1 + e^{-\eta}}{2} - e^{-\eta/2} \right) 
      + \alpha \left( -1 + \frac{1}{\eta} \left( 1 -e^{-\eta} \right) \right) 
\end{align*}
Since we want the above to be nonnegative for all $\eta > 1$, we arrive at the condition
\begin{align}
\alpha
\leq \inf_{\eta \geq 1} \left\{ 
  \frac{ \frac{1 + e^{-\eta}}{2} - e^{-\eta/2} }{ 1 - \frac{1}{\eta} \left( 1 -e^{-\eta} \right) } 
\right\}
\label{eqn:alpha-condition}
\end{align}

Plotting suggests that the minimum is attained at $\eta = 1$, with the value $\frac{1}{2} (\sqrt{e} - 1)^2$. We will fix $\alpha$ to this value and verify that 
\begin{align}
\left( \frac{1 + e^{-\eta}}{2} - e^{-\eta/2} \right) 
 + \left( \frac{1}{2} (\sqrt{e} - 1)^2 \right) \left( -1 + \frac{1}{\eta} \left( 1 -e^{-\eta} \right) \right) \geq 0 . \label{eqn:big-expression}
\end{align}
This is true with equality at $\eta = 0$. 
The derivative of the LHS with respect to $\eta$ is
\begin{align*}
\frac{1}{2} e^{-\eta} \left( e^{\eta/2} - 1 
- \frac{(\sqrt{e} - 1)^2 (e^\eta - \eta - 1)}{\eta^2} \right) .
\end{align*}
The derivative is positive at $\eta = 1$, so 0 is a candidate minimum. Eventually, $\frac{(\sqrt{e} - 1)^2 (e^\eta - \eta - 1)}{\eta^2}$ grows more quickly than $e^{\eta/2} - 1$ and surpasses the latter in value. 
The derivative is therefore negative for all sufficiently large $\eta$, and so we need only take the minimum of the LHS of \eqref{eqn:big-expression} evaluated at $\eta = 1$ and the limiting value as $\eta \rightarrow \infty$. 
We have 
\begin{align*}
\lim_{\eta \rightarrow \infty} \left( \frac{1 + e^{-\eta}}{2} - e^{-\eta/2} \right) 
 + \left( \frac{1}{2} (\sqrt{e} - 1)^2 \right) \left( -1 + \frac{1}{\eta} \left( 1 -e^{-\eta} \right) \right) = \sqrt{e} - \frac{e}{2} \geq 0
\end{align*}
Hence, \eqref{eqn:big-expression} indeed holds for $\alpha = \frac{1}{2} (\sqrt{e} - 1)^2$. We conclude that $u_\eta(-1) \geq 0$ when $\alpha \leq \frac{1}{2} (\sqrt{e} - 1)^2$. 

\subsubsection*{Putting it all together}
In the regime $\eta \leq 1$, we have the constraints $c_2 > \frac{1}{4}$ and $c_2 \geq \frac{1}{2} \frac{e - \sqrt{e}}{e - 1}$ (which exceeds $\frac{1}{4}$), so we can choose
\begin{align*}
c_1 = \frac{1}{2} - c_2 = \frac{1}{2} - \frac{1}{2} \frac{e - \sqrt{e}}{e - 1} 
= \frac{1}{2} \frac{\sqrt{e} - 1}{e - 1} = 0.1877\ldots .
\end{align*}
In the regime $\eta > 1$, we have the constraints $c_2 > \frac{1}{4}$ and $\alpha \leq \frac{1}{2} (\sqrt{e} - 1)^2 \Rightarrow c_2 \geq \frac{1}{2} - \frac{1}{2 \eta} (\sqrt{e} - 1)^2$ (which always exceeds $\frac{1}{4}$ for $\eta \geq 1$), so we can choose
\begin{align*}
c_1 = \frac{1}{2} - c_2 = \frac{\alpha}{\eta} \leq \frac{(\sqrt{e} - 1)^2}{2 \eta} = \frac{0.2104\ldots}{\eta} .
\end{align*}

The result follows by observing that in the case of $\eta \leq 1$, the supremum in \eqref{eqn:d-objective} is lower bounded by $-1 + \frac{0.18 a \eta}{n}$, and hence the optimal objective value of \eqref{eqn:our-general-moment-problem} is lower bounded by the same quantity. Therefore, the problem with the same constraints and the objective $\sup_{\mu \in [-1, 1]} \E e^{(\eta / 2) X}$ has its optimal objective value upper bounded by $1 - \frac{0.18 a \eta}{n}$. Repeat the same argument for the case of $\eta > 1$.
\end{proof}

\section{Hyper-concentrated excess losses}
\label{sec:appendix-hyper-concentrated}

\begin{lemma}
\label{lemma:hyper-concentrated}
Let $Z$ be a random variable with probability measure $P$ supported on $[-V, V]$. Suppose that $\lim_{\eta \rightarrow \infty} \E \exp(-\eta Z) < 1$ and $\E Z = \mu > 0$. Then there is a suitable modification of $Z'$ for which $Z' \leq Z$ with probability 1, the mean of $Z'$ is arbitrarily close to $\mu$, and $\E \exp(-\eta Z') = 1$ for arbitrarily large $\eta$.
\end{lemma}

\begin{proof}
First, observe that $Z \geq 0$ a.s. If not, then there must be some finite $\eta > 0$ for which $\E \exp(-\eta Z) = 1$. 
Now, consider a random variable $Z'$ with probability measure $Q_\epsilon$, a modification of $Z$ (with probability measure $P$) constructed in the following way. 
Define $A := [\mu, V]$ and $A^- := [-V, -\mu]$. Then for any $\epsilon > 0$ we define $Q_\epsilon$ as
\begin{align*}
dQ_\epsilon(z) = 
\begin{cases} 
  (1 - \epsilon) dP(z) & \text{if } z \in A \\
  \epsilon dP(-z) & \text{if } z \in A^- \\
  dP(z) & \text{otherwise} .
\end{cases}
\end{align*}

Additionally, we couple $P$ and $Q_\varepsilon$ such that the couple $(Z, Z')$ is a coupling of $(P, Q_\epsilon)$ satisfying
\begin{align*}
\E_{(Z, Z') \sim (P, Q_\epsilon)} \mathbf{1}_{\{Z \neq Z'\}} = \min_{(P', Q_\epsilon')} \E_{(Z,Z') \sim (P', Q_\epsilon')} \mathbf{1}_{Z \neq Z'} ,
\end{align*}
where the $\min$ is over all couplings of $P$ and $Q_\varepsilon$. 
This coupling ensures that  $Z' \leq Z$ with probability 1; i.e. $Z'$ is dominated by $Z$.

Now,
\begin{align}
\E \exp(-\eta Z') 
&= \int_{-V}^V e^{-\eta z} dQ_\epsilon(z) \nonumber \\
&= \int_{A^-} e^{-\eta z} dQ_\epsilon(z) + \int_A e^{-\eta z} dQ_\epsilon(z) + \int_{[0, V] \setminus A} e^{-\eta z} dQ_\epsilon(z) \nonumber \\
&= \epsilon \int_{A^-} e^{-\eta z} dP(-z) + (1 - \epsilon) \int_A e^{-\eta z} dP(z) + \int_{[0, V] \setminus A} e^{-\eta z} dP(z) \nonumber \\
&= \epsilon \int_{A} e^{\eta z} dP(z) + (1 - \epsilon) \int_A e^{-\eta z} dP(z) + \int_{[0, V] \setminus A} e^{-\eta z} dP(z) \nonumber \\
&\geq\epsilon e^{\mu \eta} P(A) + (1 - \epsilon) \int_A e^{-\eta z} dP(z) + \int_{[0, V] \setminus A} e^{-\eta z} dP(z) \label{eqn:last-line-hyper} .
\end{align}

Now, on the one hand, for any $\eta > 0$, the sum of the two right-most terms in \eqref{eqn:last-line-hyper} is strictly less than 1 by assumption. On the other hand, $\eta \rightarrow \epsilon P(A) e^{\mu \eta}$ is exponentially increasing since $\epsilon > 0$ and $\mu > 0$ (and hence $P(A) > 0$ as well) by assumption; thus, the first term in \eqref{eqn:last-line-hyper} can be made arbitrarily large for large enough by increasing $\eta$. Consequently, we can choose $\epsilon > 0$ as small as desired and then choose $\eta < \infty$ as large as desired such that the mean of $Z'$ is arbitrarily close to $\mu$ and $\E \exp(-\eta Z') = 1$ respectively.
\end{proof}

\section{Proof of VC-type results}
\label{sec:appendix-vc-type}

\subsection{Proof of \autoref*{thm:local-union}}
The localization result is a simple consequence of the following theorem.

\begin{theorem}[Local Analysis]
Let $\F \subset \real^\X$ be a separable function class for which:
\begin{itemize}
\item the constant zero function is an element of $\F$ ;
\item every function $f \in \F$ satisfies $0 \leq f \leq 1$ ;
\item $\sup_{f \in \F} \| f \|_{L_2(\Prob)} \leq \varepsilon := \frac{1}{n}$ .
\end{itemize}
Further assume that for some $\capacity \geq 1$, 
for a constant $K \geq 1$, for each $u \in (0, K]$ the $L_2(\Prob)$ covering numbers of $\F$ are bounded as
\begin{align*}
\N(u, \F, L_2(\Prob)) \leq \left( \frac{K}{u} \right)^\capacity .
\end{align*}
Then provided that $n \geq 4$ and $y > 0$, with probability at least $1 - e^{-y}$
\begin{align*}
\sup_{f \in \F} \Probn f 
\leq \frac{1}{n} \left( 
  990 \capacity \log(2 K n) + \sqrt{2 y (1 + 3960 \capacity \log(2 K n))} + \frac{2 y}{3} + 1 
\right) .
\end{align*}
\label{thm:local-analysis}
\end{theorem}

\newpage
\textsc{Remarks}
\begin{enumerate}[(i)]
\item The class $\F$ is contained in an $L_2(\Prob)$-ball of radius $\varepsilon$, and if interpreted as a loss class it is assumed that the losses are bounded.
\item Suppose the function class $\F$ is constructed by selecting from a larger class an $\varepsilon$-ball in the $L_2(\Prob)$ pseudometric around some function $f_0$ from the same larger class and taking for each function the absolute difference with $f_0$. Then the zero function trivially is in $\F$ since $|f_0 - f_0|$ is in the class. In this setup, the theorem states that with high probability there is no function in the class whose empirical risk will be ``much'' smaller/larger than the empirical risk of $f_0$. 
\end{enumerate}

\begin{proof}[Proof of \autoref{thm:local-analysis}] 
For the proof, we introduce the random variables $Z = \sup_{f \in \F} \Probn f$ and $\bar{Z} = \sup_{f \in \F} (\Probn - \Prob) f$. 
The proof is in three steps.

\subsubsection*{Step 1: Centering approximation}
It is easy to see that $Z \leq \bar{Z} + \varepsilon$, since 
\begin{align*}
Z 
= \sup_{f \in \F} \Probn f 
&= \sup_{f \in \F} (\Probn - \Prob) f + \Prob f \\
&\leq \sup_{f \in \F} (\Probn - \Prob) f 
          + \sup_{f \in \F} \Prob f \\
&\leq \bar{Z} + \sup_{f \in \F} \|f_0 - f\|_{L_1(\Prob)} \\
&\leq \bar{Z} + \sup_{f \in \F} \|f_0 - f\|_{L_2(\Prob)} \\
&\leq \bar{Z} + \varepsilon ,
\end{align*}
where the penultimate inequality follows from Jensen's inequality.

\subsubsection*{Step 2: Concentration of $\bar{Z}$ arounds its expectation}
We will apply Bousquet's version of Talagrand's inequality, appearing as equation (18) of \citep{massart2006risk} and reproduced below for convenience:
\begin{quote}
If $\G$ is a countable family of measurable functions such that, for some positive constants $v$ and $b$, one has, for every $g \in \G$, $\Prob g^2 \leq v$ and $\|g\|_\infty \leq b$, then, for every positive $y$, the following inequality holds for $W = \sup_{g \in \G} (\Probn - \Prob) g$:
\begin{align*}
\Pr\left\{ 
W - \E W \geq \sqrt{\frac{2 (v + 4 b \E W) y}{n}} + \frac{2 b y}{3 n} 
\right\} 
\leq e^{-y} .
\end{align*}
\end{quote}
We take $\G$ to be $\F$ itself; since $\F$ is separable and hence admits a countable dense subset, the countability assumption in Talagrand's inequality is not an issue. Observe that for every $f \in \F$ we have
\begin{itemize}
\item $\|f\|_\infty \leq 1$ \quad (from the range constraints on $f$)
\item $\Prob f^2 \leq \| f \|_{L_2(\Prob)} \leq \varepsilon$ \quad (by the small $L_2(\Prob)$-ball assumption on $\F$) .
\end{itemize}
Thus, taking $b = 1$ and $v = \varepsilon$, we have
\begin{align}
\Pr\left\{ 
\bar{Z} - \E \bar{Z} \geq \sqrt{\frac{2 (\varepsilon + 4 \E \bar{Z}) y}{n}} + \frac{2 y}{3 n} 
\right\} 
\leq e^{-y} . \label{eqn:talagrand-concentration}
\end{align}
It remains to bound $\E \bar{Z}$. If it can be shown to be $\tilde{O}(\frac{1}{n})$ then we will have the desired result after taking $\varepsilon = O(\frac{1}{n})$.

\subsubsection*{Step 3: Controlling the size of $\E \bar{Z}$}

Controlling $\E \bar{Z}$ can be done through chaining after passing to a symmetrized empirical process. This control is shown in \autoref{lemma:VC-type-E-sup}, stated after the current proof, yielding the result 
\begin{align}
\E \bar{Z} \leq \frac{990 \capacity \log(2 K n)}{n} .
\label{eqn:VC-type-E-sup-upper-bound}
\end{align}

\subsubsection*{Putting it all together}
The desired result follows by plugging \eqref{eqn:VC-type-E-sup-upper-bound} into the concentration result \eqref{eqn:talagrand-concentration}, incorporating the $\varepsilon$ approximation term from Step 1, and setting $\varepsilon = \frac{1}{n}$.
\end{proof}

\begin{lemma}
\label{lemma:VC-type-E-sup}
Take the same conditions as \autoref{thm:local-analysis} (Local Analysis Theorem), but instead allow that all $f \in \F$ need only satisfy $0 \leq f \leq \bound$ for some $\bound \geq 1$. Then provided that $n \geq 4$,
\begin{align*}
\E \sup_{f \in \F} (\Probn - \Prob) f 
\leq \frac{990 \capacity \bound \log(2 K n)}{n} 
\end{align*}
\end{lemma}

\begin{proof} 
To avoid measurability issues, we operate under the assumption that $\F$ is countable and in the final step of the proof apply an approximation argument.

Let $\epsilon_1, \ldots, \epsilon_n$ be iid Rademacher random variables. 
We write $\E_\epsilon$ for the expected value with respect to the random variables $\epsilon_1, \ldots, \epsilon_n$. That is, if $A$ is a random variable depending only on $\epsilon_1, \ldots, \epsilon_n, X_1, \ldots, X_n$, then $\E_\epsilon A = \E \left[ A \mid X_1, \ldots, X_n \right]$. 
Also, let $\Fcond$ be the coordinate projection of $\F$ onto the sample $\mathbf{X} = (X_1, \ldots, X_n)$:
\begin{align*}
\Fcond := \bigl\{ (f(X_1), \ldots, f(X_n) : f \in \F \bigr\} .
\end{align*}
Finally, for a set $\G \subset \real^n$ let $D(\G)$ be half of the $\ell_2$-radius of $\G$, defined as
\begin{align*}
D(\G) := \frac{1}{2} \sup_{g \in \G} \| g \|_2 ;
\end{align*}
it makes sense to refer to this as a (half) radius since we will consider $D(\F)$ and the zero function is in $\F$. 
Our life will be made easier if we use the lower bounded quantity 
$D(\G) \opvee \sigma$, for some deterministic $\sigma \leq 1$ to be chosen later. 

The first step is symmetrization. The second step is based on a chaining argument, the result of which is Corollary 13.2 of \cite{boucheron2013concentration}, restated here\footnote{\cite{boucheron2013concentration} stated this result in terms of packing numbers, but careful inspection of their proof reveals that the argument works for covering numbers as well. Moreover, other proofs generally use covering numbers.} we state the result in terms of  in a specialization to Rademacher processes for convenience:
\begin{quote}
Let $(\T, d)$ be a finite pseudometric space and $(X_t)_{t \in \T}$ be a collection of sub-Gaussian random variables. Then for any $t_0 \in \T$,
\begin{align*}
\E \sup_{t \in \T} X_t - X_{t_0} \leq 12 \int_0^{\delta / 2} \sqrt{\log \N(u, \T, d)} du ,
\end{align*}
where $\delta = \sup_{t \in \T} d(t, t_0)$. 
\end{quote}

By pushing the cardinality of $\T$ to infinity, the above result also applies to countable classes. 
As noted by \cite{boucheron2013concentration} in the paragraph concluding the statement of their Corollary 13.2, this result applies to Rademacher processes. In our case, $t_0$ will correspond to the zero function element of $\F$.

Define $\Probc := \Probn - \Prob$. Now, from symmetrization and the above chaining-based result applied to the resulting Rademacher process, we have
\begin{align*}
n \E \sup_{f \in \F} \Probc f 
&\leq 2 \E \left( \E_\epsilon \sup_{f \in \F} \sum_{j=1}^n \epsilon_j f(X_j) \right) \\
&\leq 24 \E \int_0^{D(\Fcond) \opvee \sigma} \sqrt{\log \N(u, \Fcond, \|\cdot\|_2)} du
\end{align*}
which (since if $f_{|_X}$ is the obvious coordinate projection of $f \in \F$, then $\|f_{|_X}\|_2 = \sqrt{n} \|f\|_{L_2(\Probn)}$) is at most
\begin{align*}
&24 \E \int_0^{D(\Fcond) \opvee \sigma} \sqrt{\Entropy_2 \left( \frac{u}{\sqrt{n}}, \Fcond \right) } du \\
&\leq 24 \sqrt{\capacity} \E \int_0^{D(\Fcond) \opvee \sigma} \sqrt{\log \frac{K \sqrt{n}}{u}} du ,
\end{align*}
where $\Entropy_2(u, \T) := \sup_Q \log \N(u, \T, L_2(Q))$ is the \emph{universal metric entropy} of $\T$, and in the above display we have  $\Entropy_2\left(\frac{u}{\sqrt{n}}, \Fcond \right)$ rather than $\Entropy_2(u, \Fcond)$ because we work with the $L_2(\Probn)$-norm scaled by $\sqrt{n}$. 

Making the substitution $t = u / (D(\Fcond) \opvee \sigma)$, the above is equal to
\begin{align*}
&24 \sqrt{\capacity} \E \left( \bigl( D(\Fcond) \opvee \sigma \bigr) \int_0^1 \sqrt{\log \frac{K \sqrt{n}}{t \bigl( D(\Fcond) \opvee \sigma \bigr)}} dt \right) \\
&\leq 24 \sqrt{\capacity} \E \left( 
  \bigl( D(\Fcond) \opvee \sigma \bigr) \left( \sqrt{\log \frac{K \sqrt{n}}{D(\Fcond) \opvee \sigma}} + \frac{\sqrt{\pi}}{2} \right) 
\right) \\
&\leq 24 \sqrt{\capacity} 
         \left( \sqrt{\log \frac{K \sqrt{n}}{\sigma}} + \frac{\sqrt{\pi}}{2} \right) 
         \left( \sigma + \E D(\Fcond) \right) .
\end{align*}

Now, we focus on $\E D(\Fcond)$. Observe that
\begin{align*}
\E D(\Fcond) 
&= \frac{\sqrt{n}}{2} \E \sup_{f \in \F} \sqrt{\Probn f^2} \\
&= \frac{\sqrt{n}}{2} \E \sup_{f \in \F} \sqrt{\Probc f^2 + \Prob f^2} \\
&\leq \frac{\sqrt{n}}{2} \left[ 
  \E \sup_{f \in \F} \sqrt{\Probc f^2} + \sup_{f \in \F} \sqrt{\Prob f^2} 
\right] \\
&\leq \frac{\sqrt{n}}{2} \left[  \sqrt{\E \sup_{f \in \F} \Probc f^2} + \varepsilon  \right] .
\end{align*}
where the first part of the last step follows from Jensen's inequality and the second part follows from the small $L_2(\Prob)$-ball assumption on $\F$. 
The above is at most 
\begin{align*}
\frac{\sqrt{\bound n}}{2} \left[  \sqrt{\E \sup_{f \in \F} \Probc f} + \varepsilon  \right] .
\end{align*}

Thus, putting everything together and making the replacement $\varepsilon = \frac{1}{n}$, we have
\begin{align*}
\E \sup_{f \in \F} \Probc f 
&\leq \frac{24 \sqrt{\capacity}}{n} 
         \left( \sqrt{\log \frac{K \sqrt{n}}{\sigma}} + \frac{\sqrt{\pi}}{2} \right) 
         \left( \sigma + \frac{\sqrt{\bound n}}{2} \left[  \sqrt{\E \sup_{f \in \F} \Probc f} + \frac{1}{n}  \right] 
         \right) .
\end{align*}
Finding the minimal value of $\E \sup_{f \in \F} \Probc$ just amounts to solving a quadratic equation, yielding the solution set
\begin{align*}
\sqrt{E \sup_{f \in \F} \Probc f} 
&\leq \frac{\psi \sqrt{\bound n}}{2} 
          + \sqrt{\psi \left( \sigma + \frac{\sqrt{\bound / n}}{2} \right) } 
\end{align*}
for $\psi = \frac{24 \sqrt{\capacity}}{n} \left( \sqrt{\log \frac{K \sqrt{n}}{\sigma}} + \frac{\sqrt{\pi}}{2} \right)$. 

Making the replacement $\sigma = \frac{1}{n}$, squaring, and some coarse bounding yields
\begin{align*}
\E \sup_{f \in \F} \Probc f 
\leq \frac{990 \capacity \bound \log(K n)}{n} 
\end{align*}
for $n \geq 4$, $\bound \geq 1$, and $\capacity \geq 1$.  

We now present the approximation argument to handle separable $\F$. Since $\F$ is separable, it suffices to consider a countable dense subset $\F' \subset \F$; however, a little more work is required as the covering numbers of $\F'$ may differ slightly from the covering numbers of $\F$. We now control the covering numbers of $\F'$ in terms of the covering numbers of $\F$. 
Observe that if there is an $\varepsilon$-net of $\F$ of cardinality $N$, then there is a $(2 \varepsilon)$-net of some $\F' \subset \F$ of cardinality $N$. Hence, if there is an optimal $\varepsilon$-net of $\F$ of cardinality $N$, then an optimal $(2 \varepsilon)$-net of $\F'$ has cardinality at most $N$. That is, for any probability measure $Q$ on $\X$, for any $u > 0$ we have $\N(2 u, \F', L_2(Q)) \leq \N(u, \F, L_2(Q))$. Thus, the result for separable $\F$ holds by replacing the constant $K$ with $2 K$.
\end{proof}

We now prove the localization result.

\begin{proof}[Proof of \autoref{thm:local-union}] 
First, so that we can just handle the case of functions with range $[0, 1]$, we (crudely) apply our analysis to the function class after scaling all functions by the factor $\frac{1}{\bound}$, and scale the approximation term in the last step by the factor $\bound$.\footnote{It may be possible to get a weaker dependence on $\bound$ with a more careful argument that depends on $\bound$ throughout; in particular, Talagrand's inequality can handle the parameter $\bound$.} 

For any $f_0 \in \F_\varepsilon$, observe that $\pi^{-1}(f_0)$ is the set of those functions that are covered by $f_0$ in the $L_2(\Prob)$-norm. We apply the Local Analysis Theorem (\autoref{thm:local-analysis}) to each element of the set of localized \emph{absolute difference} function classes
\begin{align*}
\left\{ \G_{f_0} : f_0 \in \F_\varepsilon \right\} ,
\end{align*}
for $\G_{f_0} := \left\{ |f_0 - f | : f \in \pi^{-1}(f_0) \right\}$. Consider an arbitrary $f_0 \in \F_\varepsilon$ and its corresponding class $\G_{f_0}$. Since $\tilde{\G}_{f_0} := \left\{ f_0 - f : f \in \pi^{-1}(f_0) \right\}$ is isomorphic to a subset of $\F$, and since any $\varepsilon$-net for $\tilde{G}_{f_0}$ trivially gives rise to an $\varepsilon$-net for $\G_{f_0}$ by taking the absolute value of each function from the original $\varepsilon$-net, it follows the $L_2(\Prob)$ covering numbers of $\G_{f_0}$ are bounded just as in \eqref{eqn:poly-covering}. 

Taking the union bound over $\F_\varepsilon$ with \autoref{thm:local-analysis} implies that with probability at least $1 - \delta$
\begin{multline*}
\max_{f_0 \in \F_\varepsilon} \sup_{f \in \pi^{-1}(f_0)} \Probn | f_0 - f | \\
\leq \frac{1}{n} \left( 
  990 \capacity \log(2 K n) + \sqrt{2 \left( \log \frac{1}{\delta} + \capacity \log(K n) \right) (1 + 3960 \capacity \log(2 K n))} + \frac{2 \left( \log \frac{1}{\delta} + \capacity \log(K n) \right)}{3} + 1 
\right) .
\end{multline*}
Ignoring the $\frac{1}{n}$ factor, the RHS is at most
\begin{align*}
990 \capacity \log(2 K n) + \sqrt{2 \left( \log \frac{1}{\delta} + \capacity \log(K n) \right) (1 + 3960 \capacity \log(2 K n))} + \log \frac{e}{\delta} + \capacity \log(K n) ,
\end{align*}
which is at most
\begin{align*}
&991 \capacity \log(2 K n) + \sqrt{2 \log \frac{1}{\delta} + 2 \capacity \log(K n) + 7920 \left( \log \frac{1}{\delta} \right) \capacity \log(2 K n) + 7920 (\capacity \log(2 K n))^2} + \log \frac{e}{\delta} \\
&\leq 1080 \capacity \log(2 K n) + \sqrt{2 \log \frac{1}{\delta} + 2 \capacity \log(K n) + 7920 \left( \log \frac{1}{\delta} \right) \capacity \log(2 K n)} + \log \frac{e}{\delta} \\
&\leq 1080 \capacity \log(2 K n) + 90 \sqrt{\left( \log \frac{1}{\delta} \right) \capacity \log(2 K n)} + \log \frac{e}{\delta} ,
\end{align*}
where the last inequality holds provided that $\delta$ is not too large; it suffices to assume $\delta \leq \frac{1}{2}$.
\end{proof}

Finally, we prove the fast rates exact oracle inequality for VC-type classes.

\subsection{Proof of \autoref*{thm:vc-type-fast-rates}}

\begin{proof}[Proof of \autoref{thm:vc-type-fast-rates}] 
For convenience, we begin by abusing notation and redefining $\F$ as $\F := \lossclass$; the abuse includes $f^*$ being redefined as $\lossof{f^*}$. With these abuses, for any $f \in \F$ we redefine $Z_f$ as $Z_f := f - f^*$. 

Next, we introduce a few subclasses that will be in play. Recall that for any $\gamma_n > 0$, $\F_{\succeq \gamma_n}$ is the subclass of $\F$ for which the excess risk is at least $\gamma_n$. 
Also, for any $\gamma_n > 0$, let $\F_{\succeq \gamma_n, \varepsilon}$ be a proper cover of $\F_{\succeq \gamma_n}$ with respect to the $L_2(\Probn)$ norm, with $\varepsilon = \frac{1}{n}$. 
For each $\eta > 0$ and $\F_{\succeq \gamma_n, \varepsilon}$, let $\F_{\succeq \gamma_n, \varepsilon}^{(\eta)} \subset \F_{\succeq \gamma_n, \varepsilon}$ correspond to those functions for which $\eta$ is the largest constant such that 
$\E \exp(-\eta Z_f) = 1$. 
After making the same implicit change to $\F_{\succeq \gamma_n, \varepsilon}$ for ``hyper-concentrated'' excess loss random variables (i.e. those $Z_f$ for which $\lim_{\eta \rightarrow \infty} \E \exp(-\eta Z_f) < 1$) as was made to $\F_{\succeq \gamma_n}$ in the proof of \autoref{thm:finite-fast-rates}, we have $\F_{\succeq \gamma_n, \varepsilon} = \bigcup_{\eta \in [\eta^*, \infty)} \F_{\succeq \gamma_n, \varepsilon}^{(\eta)}$. 

Let $\gamma_n = \frac{a}{n}$ for some constant $a > 1$ to be fixed later. 
Consider an arbitrary $\eta \in [\eta^*, \infty)$ for which $| \F_{\succeq \gamma_n, \varepsilon}^{(\eta)} | > 0$, and recall that all functions $f$ in this class satisfy $\E Z_f \geq \frac{a}{n}$. 
Individually for each such function $f$, we will apply the Cram\'er-Chernoff Theorem (\autoref{thm:cramer-chernoff}) as follows. From the Bounded Losses Lemma (\autoref{lemma:bounded-losses}), 
we have $\Lambda_{-Z_f}(\eta / 2) = \Lambda_{-\frac{1}{\bound} Z_f}(\bound \eta / 2)$. From the Stochastic Mixability Concentration Theorem (\autoref{thm:stochastic-mixability-concentration}), the latter is at most
\begin{align*}
-\frac{0.18 (\bound \eta \opwedge 1) (a / \bound)}{n} 
= -\frac{0.18 \eta a}{(\bound \eta \opvee 1) n} .
\end{align*}
Hence, the Cram\'er-Chernoff Theorem (\autoref{thm:cramer-chernoff}) with $t = \frac{a}{2 n}$ and the $\eta$ from that Theorem taken to be $\eta / 2$ implies that the probability of the event $\Probn f \leq \Probn f^* + \frac{a}{2 n}$ is at most
\begin{align*}
\exp\left( -0.18 \frac{\eta}{\bound \eta \opvee 1} a + \frac{a \eta}{4 n} \right) 
= \exp\left( - \eta a \left( \frac{0.18 }{\bound \eta \opvee 1} - \frac{1}{4 n} \right) \right) .
\end{align*}

Applying the union bound over all of $\F_{\succeq \gamma_n, \varepsilon}$, we conclude that
\begin{align*}
\Pr \left\{ 
  \exists f \in \F_{\succeq \gamma_n, \varepsilon} : 
  \Probn f \leq \Probn f^* + \frac{a}{2 n} 
\right\} 
\leq \left( \frac{K}{\varepsilon} \right)^\capacity \exp \left( -\eta^* a \left( \frac{0.18}{\bound \eta^* \opvee 1} - \frac{1}{4 n} \right) \right) .
\end{align*}

Now, observe that if we consider some fixed failure probability $\frac{\delta}{2}$ and invert to obtain the corresponding $a$, we have
\begin{align}
a 
= \frac{\capacity \log \frac{K}{\varepsilon} + \log \frac{2}{\delta}}{\eta^* \left( \frac{0.18}{\bound \eta^* \opvee 1} - \frac{1}{4 n} \right)} 
&= \frac{\capacity \log \frac{K}{\varepsilon} + \log \frac{2}{\delta}}{\eta^* \left( \frac{0.18 - (\bound \eta^* \opvee 1) / (4 n)}{\bound \eta^* \opvee 1} \right)} \nonumber \\
&\leq 
\frac{\left( \bound \eta^* \opvee 1 \right) \left( \capacity \log \frac{K}{\varepsilon} + \log \frac{2}{\delta} \right)}{\eta^* \left( 0.18 - \frac{1}{4 n} \right)} \nonumber \\
&\leq 
8 \left( \bound \opvee \left( \frac{1}{\eta^*} \right) \right) \left( \capacity \log \frac{K}{\varepsilon} + \log \frac{2}{\delta} \right) =: \lambda  \label{eqn:upper-bound-choice-of-a} ,
\end{align}
for $\gamma^{(1)}_n := \frac{\lambda}{n}$, 
where the last inequality holds since $n \geq 5$. Note that by instead setting $\frac{a}{n} = \gamma^{(1)}_n$ (defined in \eqref{eqn:upper-bound-choice-of-a}) the failure probability can only decrease. 
Thus, for any $\gamma_n \geq \gamma^{(1)}_n$, we have
\begin{align*}
\Pr \left\{ 
  \exists f \in \F_{\succeq \gamma_n, \varepsilon} : 
  \Probn f \leq \Probn f^* + \frac{\gamma_n}{2} 
\right\} 
\leq \frac{\delta}{2} .
\end{align*}

Next, we control the behavior of the subclass $\F_{\succeq \gamma_n} \setminus \F_{\succeq \gamma_n, \varepsilon}$. From \autoref{thm:local-union}, if $\delta \leq \frac{1}{2}$
\begin{align*}
\Pr \Bigl\{ \exists f \in \F_{\succeq \gamma_n} : 
\Probn f 
< \Probn \pi(f) - \gamma^{(2)}_n
\Bigr\} \leq \frac{\delta}{2} .
\end{align*}
for $\gamma^{(2)}_n = \frac{\bound}{n} \left( 
  1080 \capacity \log(2 K n) 
  + 90 \sqrt{\left( \log \frac{2}{\delta} \right) \capacity \log(2 K n)} 
  + \log \frac{2e}{\delta} 
\right)$.

Now, combining the above two high probability guarantees, with probability at least $1 - \delta$ both statements below hold for all $f \in \F_{\succeq \gamma_n}$:
\begin{align*}
&\Probn f \geq \Probn \pi(f) - \gamma^{(2)}_n \\
&\Probn \pi(f) \geq \Probn f^* + \frac{\gamma_n}{2} .
\end{align*}
Thus, with the same probability, for all $f \in \F_{\succeq \gamma_n}$:
\begin{align*}
&\Probn f \geq \Probn f^* + \frac{\gamma_n}{2} - \gamma^{(2)}_n .
\end{align*}

Setting $\gamma_n = (\gamma^{(1)}_n \opvee 2 \gamma^{(2)}_n) + \frac{1}{n}$, and recalling that ERM selects hypotheses purely based on their empirical risk, we see that with probability at least $1 - \delta$, ERM will not select any hypothesis whose excess risk is at least
\begin{align*}
(\gamma^{(1)}_n \opvee (2 \gamma^{(2)}_n)) + \frac{1}{n} .
\end{align*}
\end{proof}

\end{document}